\documentclass[twoside,12pt, reqno]{amsart}

\usepackage[foot]{amsaddr}

\usepackage{amssymb,amsmath,amsfonts,amsthm,mathtools,color}

\usepackage{mathrsfs}
\usepackage{bm} % for bm macro

\usepackage{comment} % for comment large section
\usepackage[title,titletoc,header]{appendix}
\usepackage{graphicx,subfig}

\usepackage{paralist}

\usepackage{tikz}
\usetikzlibrary{arrows,positioning,shapes.geometric}

\usepackage[normalem]{ulem}

\usepackage[linesnumbered, ruled,vlined]{algorithm2e}
\SetKwInput{KwInput}{Input}                % Set the Input

\graphicspath{ {./figures/} }

\usepackage{indentfirst}
\usepackage{multicol}
\usepackage{booktabs}
\usepackage{url}
\usepackage[outdir=./]{epstopdf} %to enable the use of eps

\usepackage[shortlabels]{enumitem}
\usepackage{hyperref}
\hypersetup{
    colorlinks=true, %set true if you want colored links
    linktoc=all,     %set to all if you want both sections and subsections linked
    linkcolor=blue,  %choose some color if you want links to stand out
}
\numberwithin{equation}{section}

\newtheorem{Theorem}{Theorem}[section]
\newtheorem{Lemma}[Theorem]{Lemma}
\newtheorem{Proposition}[Theorem]{Proposition}
\newtheorem{Assumption}{H.\!\!}

\theoremstyle{definition}

\newtheorem{Example}{Example}[section]

\theoremstyle{remark}
\newtheorem{Remark}{Remark}[section]

\def\to{\rightarrow}

\def\cA{\mathcal{A}}

\def\cE{\mathcal{E}}
\def\cF{\mathcal{F}}

\def\cH{\mathcal{H}}

\def\cN{\mathcal{N}}
\def\cO{\mathcal{O}}
\def\cP{\mathcal{P}}

\def\cX{\mathcal{X}}
\def\cY{\mathcal{Y}}

\def\d{{\mathrm{d}}}

\def\sE{{\mathbb{E}}}
\def\sF{{\mathbb{F}}}

\def\sH{{\mathbb{H}}}

\def\sN{{\mathbb{N}}}
\def\sP{\mathbb{P}}

\def\sR{{\mathbb R}}
\def\sS{{\mathbb{S}}}

\newcommand{\tr}{\textnormal{tr}}
\newcommand{\op}{\textnormal{op}}

\DeclareMathOperator*{\argmin}{arg\,min}

\newcommand{\lc}
{\mathrel{\raise2pt\hbox{${\mathop<\limits_{\raise1pt\hbox
{\mbox{$\sim$}}}}$}}}

\newcommand{\gc}
{\mathrel{\raise2pt\hbox{${\mathop>\limits_{\raise1pt\hbox{\mbox{$\sim$}}}}$}}}

\newcommand{\ec}
{\mathrel{\raise2pt\hbox{${\mathop=\limits_{\raise1pt\hbox{\mbox{$\sim$}}}}$}}}

\def\bb{\begin{equation}} \def\ee{\end{equation}}
\def\bbn{\begin{equation*}} \def\een{\end{equation*}}

\def\beqn{\begin{eqnarray}}  \def\eqn{\end{eqnarray}}

\def\beqnx{\begin{eqnarray*}} \def\eqnx{\end{eqnarray*}}

\def\bn{\begin{enumerate}} \def\en{\end{enumerate}}

\def\bd{\begin{description}} \def\ed{\end{description}}



\title{$\epsilon$-Policy Gradient for Online Pricing  
}

\author{Lukasz Szpruch$^{1,2}$}
\address{$^1${School of Mathematics, University of Edinburgh}}
\address{$^2${Alan Turing Institute}}
\email{L.Szpruch@ed.ac.uk}
\author{Tanut Treetanthiploet$^{3,4}$}
\address{$^3${The Institute for Fundamental Study, Naresuan University}}
\address{$^4${Quantum Technology Foundation (Thailand)}}
\email{ttreetanthiploet@gmail.com}
\author{Yufei Zhang$^{5}$}
\address{$^5${Department of Mathematics, Imperial College London, London, UK}}
\email{yufei.zhang@imperial.ac.uk}



\date{}

\begin{document}

\maketitle

\begin{abstract}
  Combining model-based and model-free reinforcement learning  approaches,  this paper proposes  and analyzes
  an $\epsilon$-policy gradient 
  algorithm  for the online pricing learning task. The algorithm extends $\epsilon$-greedy algorithm by replacing greedy exploitation with gradient descent step and facilitates learning via model inference.
  We optimize the regret of the proposed algorithm by quantifying the exploration cost in terms of the exploration probability $\epsilon$ and the exploitation cost in terms of the gradient descent optimization and gradient estimation errors. The algorithm achieves an expected regret of order $\mathcal{O}(\sqrt{T})$ (up to a logarithmic factor) over $T$ trials.

\end{abstract}

%\tableofcontents 
 
\section{Problem formulation}
\label{sec:intro}

Model-based and model-free learning are two prominent approaches in the field of reinforcement learning (RL), each with its own set of advantages and disadvantages. Model-based are sample efficient, requiring fewer interactions with the environment thus obviating the need for costly Monte Carlo rollouts \cite{sutton2018reinforcement}. Furthermore, a model can be used to generalize across various but related tasks, enhancing its applicability in diverse contexts. Importantly the use of a model enhances decisions interpretability and auditability. Despite these benefits, model-based learning is not without drawbacks. Model misspecification can lead to poor performance and the complexity and computational cost involved in developing and maintaining an accurate model might be significant. 
  Conversely, model-free approaches are generally simpler and scale effectively, as they learn policies directly, often through the use of efficient gradient based algorithms, bypassing the burden associated with environmental modeling. However, model-free learning require extensive interaction with the environment, which can be both costly and time-consuming. Additionally, the lack of a model means that when the objective or environment changes, one often need to start learning task from scratch.  In this paper we are interested in the following natural  question:

\vspace{0.3cm}
\emph{Can one integrate  model-based and model-free methods to provably achieve the benefits of both?} 
\vspace{0.3cm}

In this paper, we answer this question affirmatively within the context of online pricing problems by developing $\epsilon$-policy gradient ($\epsilon$-PG) algorithm which extends $\epsilon$-greedy algorithm by replacing greedy exploitation with gradient descent step and facilitates learning via model inference.

\subsection*{Motivating example}

To motivate
 our learning problem,
 let us consider the following dynamic pricing problem commonly encountered by many online retailers.
In this problem, 
    customers arrive sequentially to purchase a product offered by the agent, with each customer belonging to a specific segment. For each customer, the agent selects a price from the set of all admissible prices and observes  whether the offered price is accepted by the customer.
    The customer's response  is random and follows from   an \emph{unknown} 
 distribution   depending  on the customer segment  and the quoted price. 
If the customer accepts the offered price, the agent receives a reward based on the customer segment and the offered price; otherwise, the agent's reward is zero.
The agent's objective is to learn a pricing policy that maps customer segments to optimal quoted prices while simultaneously maximizing cumulative reward.

This example highlights several common features in online pricing problems: (1) The agent's reward for each trial depends on the agent's action,  and the customer's segment and response. (2) The customer's responses for a given quotation are  unknown and constitute the main source of randomness in the realized reward. (3) If the offered price is accepted by a customer, the precise dependence of the realized reward on the offered price and customer segment is typically a known function to the agent, as it depends on the associated cost of delivering the sold product and the strategic plan of the company. These  stylized features are the basis of our mathematical formulation of a learning framework, which we describe next.

\subsection*{Problem formulation}
Let $(\cX,\Sigma_{\cX})$ be a measurable space  representing  the feature space, let 
$(\cA,\Sigma_{\cA})$ be a measurable space  representing  the agent's admissible actions, and 
let 
 $(\cY, \Sigma_\cY)$ be a measurable space 
consisting of  all potential  responses.
At each time step $t\in \mathbb{N}$, a feature $x_t \in \cX$ is sampled from an unknown distribution $\mu\in \cP(\cX)$ and revealed to the agent. The agent chooses an action $a_t\in \cA$ based on the feature $x_t$ and all historical observations. The agent observes the corresponding response $y_t\in \cY$ to her action, which follows from some unknown conditional distribution $\nu (\d y|x_t,a_t)\in \cP( \cY)$.
The resulting instantaneous reward 
is given by $r(x_t, a_t, y_t)$
for a known function $ r:\cX\times\cA\times \cY\to \sR$. 
Precise knowledge of   $r$, $\nu$ and $\mu$ available to the agent  will be 
  given in Section \ref{sec:main_result}.

To measure the performance of the agent's actions, 
we define the expected reward $\bar r:\cX\times \cA\to \sR$ for a given feature $x\in \cX$ and action $a\in \cA$:  
\begin{equation}
\label{eq:expected_reward}
\bar r  ( x, a) \coloneqq \int_{\cY} r\big( a, x, y \big) \nu \big( \d y|x, a  \big),
\end{equation}
where we average the realized  reward with respect to the randomness in the response distribution.
For each $T\in \sN$,   we define the  regret  of the agent's  actions $(a_t)_{t=1}^T$ by:
\begin{equation}
    \label{eq:regret}
    \textrm{Reg}\left((a_t)_{t=1}^T\right) \coloneqq \sum_{t=1}^T \Big(  \sup_{a \in \cA}\bar{r} ( x_t, a) -\bar{r} ( x_t, a_t) \Big),
\end{equation}
which compares the expected reward of   action $a_t$ against 
the optimal reward 
$\sup_{a \in \cA}\bar{r} ( x_t, a)$  for 
the context $x_t$ at  each time point $t\in \sN$.
 The regret $  \textrm{Reg}(\cdot)$ characterizes
the cumulative loss from taking sub-optimal actions up to the $T$-th trial.
Agent's aim is to construct
actions whose regret grows sublinearly with respect to $T$.

\subsection*{Our work}

This paper   proposes  and analyzes
  an $\epsilon$-policy gradient 
  ($\epsilon$-PG)
  algorithm  for  the above learning task 
  by integrating model-based and model-free RL approaches. 
The central component of the algorithm is a PG method, which  updates 
 the  pricing policy   $\phi: \mathcal{X} \to \mathcal{A}$  
 using  the  gradient  of the expected reward \eqref{eq:expected_reward} at the current policy $\phi$. 
The  unknown policy gradient $\nabla_a \bar r$ is not evaluated using  standard  black-box  Monte-Carlo gradient estimation techniques (e.g., REINFORCE method \cite{williams1992simple}), which often suffer from %high variance and hence 
slow convergence and hence are prohibitively costly for many pricing problems.
 Instead, 
  our algorithm leverages a model-based approach to enhance sample efficiency in the gradient evaluation and to avoid cold-start issues. 

Specifically, we assume the  distribution $\nu(dy|x,a)$ 
of the response variable follows  a parametric form   $\pi_{\theta^\star}(dy|x,a)$  with an unknown parameter $\theta^\star$. 
After each trial $t = 1, 2,\ldots$, we estimate $\theta^\star$ by solving an empirical risk minimisation problem using historical observations. 
The policy is then updated  based on the estimated gradient $\nabla_a \bar r_{\theta_t}$, where $\bar r_{\theta_t}$ denotes the reward function \eqref{eq:expected_reward} with the estimated response distribution $\pi_{\theta_t}$, using policy gradient with a learning rate $\eta_t$ of the form
\begin{equation*}
     \phi_{t}=\phi_{t-1}+\eta_t (\nabla_a \bar r_{\theta_{t}})(\cdot,\phi_{t-1}(\cdot)) \quad t = 1, 2,\ldots\,.
\end{equation*}
Note that for a given parameter $\theta_t$, the gradient $\nabla_a \bar r_{\theta_t}$ can be computed with no variance.
 To ensure the asymptotic consistency of the gradient evaluation, 
 an exploration strategy is exercised with 
 probability $\epsilon$ to   explore the parameter space,
while the exploration probability $\epsilon$ is   reduced at a suitable rate as the learning proceeds.

 We optimize the regret of the proposed $\epsilon$-PG algorithm by quantifying the exploration and exploitation costs. Our analysis accommodates a general parametric model $\pi_\theta$ of $\nu$ and identifies the intricate interplay between the   loss function in the empirical risk minimization    and the structure of $\pi_\theta$. 
 This subsequently facilitates   quantifying the exploration cost in terms of the exploration probability $\epsilon$ used in the $\epsilon$-PG strategy.
 The exploitation cost is quantified  in terms of the optimization error resulting from gradient descent and the gradient estimation error resulting from   the inexact response probability. By optimizing the exploration probability and learning rate in the policy gradient update, our algorithm achieves an expected regret of order $\mathcal{O}(\sqrt{T})$ (up to a logarithmic factor) over $T$ trials.

{
\subsection*{Our approaches and most related works}

To the best of our knowledge, this is the first   theoretical work on 
 regret bounds of a policy gradient algorithm for online pricing problems with general feature and action spaces.  
 In the sequel, we will compare the proposed algorithm with existing algorithms for two closely related learning tasks: contextual bandits and online   optimizations.

 \subsubsection*{Contextual bandits algorithms}
 
% The proposed algorithm enhances the efficiency of off-the-shelf contextual bandit algorithms for the   online pricing task under consideration.

  Most contextual bandit algorithms iteratively estimate the expected reward $\bar r$ (without using the structure \eqref{eq:expected_reward}) based on realized rewards from previous trials and exercise the greedy policy which maximizes the estimated expected rewards over all actions (see, e.g., \cite{filippi2010parametric, agrawal2013thompson, chowdhury2017kernelized, hao2020adaptive, vakili2021information, janz2023exploration} and references therein). However, there are three drawbacks that make these algorithms unsuitable for the considered online pricing problems: (1) Computing the greedy policy at each trial can be expensive, especially for large action space $\cA$. (2) The algorithm's performance relies on the efficiency of the estimation oracle for the expected reward $\bar r$. Such sample efficiency is typically quantified under the condition that $r$ satisfies a linear model \cite{agrawal2013thompson, hao2020adaptive}, a generalized linear model  \cite{filippi2010parametric, janz2023exploration}, or a Gaussian process model \cite{chowdhury2017kernelized, vakili2021information}. Unfortunately, these structural assumptions of $r$ may not hold for practical online pricing problems due to nonlinearity of  $\bar r$; see \cite{cohen2021generalised} for a concrete example where   generalized linear models  fail to represent the expected reward due to the lack of monotonicity of $\bar r$. (3) The algorithm is not sample efficient with changing objectives. Indeed, when the instantaneous reward $r$ changes, e.g., due the change of the cost for  delivering  the product,  the expected reward $\bar r$ changes as well. Consequently, one often needs to restart the learning task from scratch.

Our work tackles the aforementioned drawbacks by utilizing the decomposition \eqref{eq:expected_reward} of the expected reward and directly learning the response distribution, a methodology more aligned with industrial practices \cite{blier2020machine}. This approach avoids potential model misspecification in reward estimation and allows for imposing explicit and interpretable models of the response distribution, thereby facilitating  the development of effective learning algorithms. Additionally, it ensures the algorithm adapts quickly to changes in the instantaneous reward $r$, as we can utilize the previously learned response distribution to initialize the learning task. Furthermore, we enhance the algorithm's efficiency by replacing greedy exploitation with policy gradient exploitation, which is more efficient, particularly for continuous action spaces.

 \subsubsection*{Online  optimization  algorithms}

   The proposed $\epsilon$-PG algorithm can be viewed as an online gradient descent approach for maximizing the unknown expected reward over an infinite-dimensional policy space. Due to this infinite-dimensional nature, existing convergence analysis for online gradient descent in finite-dimensional problems (see e.g.~\cite{hazan2016introduction}) is not suitable for the $\epsilon$-PG algorithm. Furthermore, the $\epsilon$-PG algorithm updates using biased gradients derived from an inexact response distribution. This introduces technical  challenges in the regret analysis of $\epsilon$-PG. Unlike existing online gradient descent algorithms that work with unbiased gradients \cite{hazan2016introduction} or biased gradients whose biases diminish at prescribed rates \cite{ajalloeian2020convergence}, our setting requires agents to actively manage the magnitude of gradient evaluation errors through strategic exploration. Consequently, optimizing the algorithm's regret involves a critical balance between exploration and exploitation.

\subsection*{Notation}
For a measurable space $\cX$, we denote by $\cP(\cX)$ the space of probability measures on $\cX$.
For each $d\in \sN$,   we denote by $I_d$ the $d\times d$  identity matrix, by $\sS^d$, ${\sS}^d_{\ge 0}$
and  $\sS^d_+$ 
 the space of $d\times d$  symmetric, symmetric positive semidefinite, and symmetric positive definite 
 matrices, respectively, and by 
$\rho_{\max}(A)$
 and $\rho_{\min}(A)$
 the largest and smallest eigenvalues of $A\in \sS^d$.  We equip $\sS^d$ with the Loewner (partial) order such that for each $A,B\in \sS^d$, $A\succeq B$  if $A-B\in   {\sS}^d_{\ge 0}$. 
 For a given $A\in \sR^{d\times d}$, 
we denote by $\|A\|_{\op}$ its spectral norm   or equivalently   the operator norm induced by  Euclidean norms.

\section{Main results} 
 \label{sec:main_result}
 
 This section introduces the precise assumptions of the model coefficients,
 derives the policy gradient algorithm and analyzes its regret. 
To facilitate the analysis, we assume without loss of generality that all random variables are supported on a probability space $(\Omega, \sF,\sP)$, which can be defined as the product space of probability spaces associated with each observation.

 \subsection{Estimation of response distribution}

This section studies the estimation of the response distribution, which is crucial for evaluating the policy gradient. We assume a generic parametric model $\pi_{\theta^\star}$ for $\nu  $ and identify the essential structure of $\pi_{\theta^\star}$ that facilitates efficient learning of $\theta^\star$ using an empirical risk minimization. Specifically, we impose the following model assumption  for $\nu$.

\begin{Assumption}
\label{assum:pi_theta}
There exists %a known convex set $\Theta\subset \sR^d$,
a  known function $\pi:\sR^d\times  \cX\times \cA\to \cP( \cY)$
such as 
 $\nu(\d y|x,a) = \pi_{\theta^\star}(\d y|x,a)$
for some unknown $\theta^\star \in \sR^d$,
and $y\mapsto r(x,a,y)$ is integrable under the measure $\pi_\theta(\d y|x,a)$
for all $(\theta, x,a)\in \sR^d\times \cX\times \cA$.

\end{Assumption}

We further assume for simplicity  that  the agent estimates $\theta^\star$ 
through an empirical risk minimization.  
Specifically,
after the $T$-th trials,
given the observations  $(x_t,a_t,y_t)_{t=1}^T$,
 the agent  estimates   $\theta^\star$ by   solving an empirical risk minimisation problem:
% and then    projecting the estimated parameter onto the parameter set $\Theta$: 
\begin{align}
\label{eq:ls}
{\theta}_T \coloneqq \argmin_{\theta\in \sR^d} \left( \sum_{t=1}^T \ell ({\theta}, x_t, a_t, y_t  ) +\frac{1}{2} \|\theta \|^2_{\sR^d}\right),
%\quad 
%{\theta}_T =  \argmin_{\theta\in \Theta}\|\theta-{\theta}_T\|^2_{\sR^d}.
\end{align}
where   $\ell: \sR^d\times  \cX\times \cA\times \cY\to \sR$ 
is a suitable loss function satisfying the following conditions.

\begin{Assumption}
\label{assum:loss}
$\ell$ is compatible with  $\pi_\theta$ in (H.\ref{assum:pi_theta}) in the sense that, 
there exists   $H : \cX \times \cA \to \sS^{d}_{\ge 0}$, called an information matrix,
%  and a Markov   kernel 
%  $\pi_{\textsf{Exp}}:\cX\to   \cP(\cA)$
   such that for all $(x,a, y ) \in    \cX \times \cA\times \cY$,  
 \begin{enumerate}[(1)]
 \item \label{item: convex} 
  $\ell(\cdot, x,a,y) $  is twice differentiable   and satisfies   $  \nabla^2_\theta \ell(\theta, x,a,y)  \succeq H(x,a) $ for all $\theta\in \sR^d$; 
  \item
    \label{item: compatibility}   
  $ 
           \int_\cY \exp \left( \lambda^\top \nabla_\theta  \ell (\theta^\star, x,a, y ) \right) \pi_{\theta^\star } (\d y | x,a) \le \exp \big(\lambda^\top H(x,a) \lambda \big)
$      for all   $\lambda \in \sR^{d}$,
where $\theta^\star$ is the (unknown) true parameter in (H.\ref{assum:pi_theta}).
 \end{enumerate}
\end{Assumption}

Assumption (H.\ref{assum:loss}) highlights  the essential properties of the loss function $\ell$  
that allow for 
quantifying the error ${\theta}_T-\theta^\star$ in high probability. 
These conditions generalize the results from  the special case  where $\ell$ represents the logarithm of $\pi_\theta$'s density and  the estimator \eqref{eq:ls} corresponds to     the (regularized) maximum likelihood estimator. 
 In this context, it is known that the Hessian $\nabla^2_\theta \ell$
 is the   (observed) Fisher information,
 and the gradient $\nabla_\theta \ell (\theta^\star,\cdot )$
 asymptotically     Gaussian distributed     with  a  variance being the Fisher information (see e.g., \cite[Section 8.12]{malinvaud1970statistical}). 
Here we relax the conditions by assuming that
the gradient  $\nabla_\theta \ell (\theta^\star, x,a, \cdot )$ %under the true response distribution $\pi_{\theta^\star}(\d y|x,a)$ 
is   sub-Gaussian distributed, and its tail behavior can be quantified by a lower bound 
 $H(x,a)$ 
of 
  the Hessian matrix $\nabla^2_\theta \ell$.

Here we give a concrete   example of  loss functions satisfying   (H.\ref{assum:loss})
 where $\pi_\theta$ is given by  a (feature-dependent) generalized linear model.

 \begin{Example} 
\label{example:negative log loss}
Suppose that 
 $\pi:\sR^d\times  \cX\times \cA\to \cP( \cY)$ in (H.\ref{assum:pi_theta}) is of the form
\begin{equation}
\label{eq:glm}
 \pi_{\theta}(\d y|x,a) = g(x,y) \exp\left( h(x,y)^\top  \psi(x,a) {\theta} - b\big(x,  \psi(x,a) {\theta} \big) \right)\overline{\nu}(\d y),
\end{equation}
 where $\overline{\nu}$ is a given reference measure on $\cY$,
and  
  $g: \cX \times \cY \to \sR$,  $ h : \cX \times \cY \to \sR^{m}$, 
  $\psi : \cX \times \cA \to \sR^{m \times d}$,
  and  $b : \cX \times \sR^m \to \sR$
 are given functions 
 such that 
  $  \int_{\cY}  g(x,y) \exp\left( h(x,y)^\top w \right) \overline{\nu} (\d y) = \exp \big(  b(x,w )\big)$
    for all $w \in \sR^m $.
 
 Assume    that  there exists a subset 
 $\cH \subseteq \sR^m$ such that 
 $ \psi(x,a) {\theta}^\star \in \cH$ for all $( x,a) \in \cX\times \cA$.
 Assume further that 
there exists a function $\tilde b : \cX \times \sR^m\to \sR$ 
 and constants 
  $c_1, c_2 > 0$ 
 such that 
for all $x\in \cX$, 
$\tilde{b} (x,w)= {b} (x,w)$ for all $w\in \cH$, 
 $w \mapsto   (\tilde b(x,w) ,   b(x,w))$ is twice differentiable,
and  $  \nabla^2_w \tilde b(x,w) \succeq c_1 I_m $
 and $  \nabla^2_w b(x,w) \preceq c_2 I_m$ for all $w\in \sR^d$.
Then    
  the loss function  $\ell: \sR^d\times \cX\times \cA\times \cY \to \sR$ 
defined by 
\begin{equation}
\label{eq:loss_glm}
    \ell (\theta, x,a,y ) =- \frac{2c_1}{c_2} \left(h(x,y)^\top  \psi (x,a) {\theta} - \tilde b \big(x, \psi(x,a) {\theta} \big)\right),
\end{equation}
satisfies (H.\ref{assum:loss})  with $H(x,a) =  \frac{2c_1^2}{c_2} \psi(x,a)^\top \psi(x,a)$.
\end{Example}
 
The proof of Example \ref{example:negative log loss} is given in Section  \ref{sec:log-likelihood}.

\begin{Remark}

In the special case where $h$ is linear in $y$, $h$, $\psi$, and $b$ are independent of $x$, and $\cA$ is a finite set, the model \eqref{eq:glm} aligns with the generalized linear model for multi-armed bandits studied in \cite{filippi2010parametric}. 
  This family contains commonly used statistical models for response distributions
   as outlined in \cite{blier2020machine}, including the Gaussian and Gamma distributions when the reference measure $\overline{\nu}$ is the Lebesgue measure, and the Poisson and Bernoulli distributions when $\overline{\nu}$ is the counting measure on the integers.

To facilitate the analysis, we assume the agent has access to the magnitude $\cH$ of the kernel $\psi(x,a)\theta^\star$, and  construct a compatible log-likelihood loss function \eqref{eq:loss_glm} by extending $b$ with an arbitrary strongly convex function outside the set $\cH$. Note that  to analyze algorithm regrets, it's common to assume some a-priori information on the true parameter. For instance, \cite{filippi2010parametric, szpruch2021exploration, basei2022logarithmic, guo2023reinforcement, szpruch2024optimal} assumes the range of the unknown parameter and \cite{gao2022logarithmic,gao2022square} assumes known upper/lower bounds for holding times of continuous-time Markov chains.
 
\end{Remark}

Assumption (H.\ref{assum:loss}) also includes 
nonlinear least-squares loss functions
for suitable bounded response variables,
as shown in the following example. 
Given that, in most online pricing problems, the response variable models whether a given quotation is accepted by the customer as discussed in Section  \ref{sec:intro}, we simplify our presentation by considering real-valued response variables. Similar results can be naturally extended to multivariate response variables.

\begin{Example}
\label{ex:ls}
    Suppose that $\cY\subset [\underline{y},\overline{y}] $ for some $-\infty< \underline{y}<\overline{y}< \infty$,
    and 
$\pi:\sR^d\times  \cX\times \cA\to \cP( \cY)$  in  (H.\ref{assum:pi_theta}) has first moments, i.e., 
 $\mu(\theta,x,a) \coloneqq \int_\cY y \pi_\theta(\d y|x,a) $ is well-defined  for all $(\theta, x,a)\in \sR^d\times \cX\times \cA$.  

Assume that $\mu$ is twice differentiable in $\theta$, 
and there exists   $\mathbb H:\cX\times \cA\to \sR^d_{\ge 0}$
and constants $c_1, c_2>0$
such that 
$ c_1 \le  1/( \overline{y}   -\underline{y})$,
$(\nabla^2_\theta \mu) (\theta, x,a)\preceq  \sH(x,a)\preceq c_1 (\nabla_\theta \mu) (\theta, x,a)(\nabla_\theta \mu) (\theta, x,a)^\top$, 
and $(\nabla_\theta \mu) (\theta^\star, x,a)(\nabla_\theta \mu) (\theta^\star, x,a)^\top\preceq c_2 \sH(x,a)$.
Then    
  the loss function  $\ell: \sR^d\times \cX\times \cA\times \cY \to \sR$ 
defined by 
\begin{equation}
\label{eq:loss_ls}
    \ell (\theta, x,a,y ) =  \frac{C_1}{2} \left(y -  \mu  (\theta, x, a  )\right)^2,
    \quad \textnormal{with $C_1 \coloneqq 
  \frac{8}{c_2(\overline{y}- \underline{y})^2}\left(\frac{1}{c_1}+\underline{y}  - \overline{y}\right)
 $}
\end{equation}
satisfies (H.\ref{assum:loss})  with $H(x,a) =  C_1\left(\frac{1}{c_1}+\underline{y}-  \overline{y}\right) \sH(x,a) $.

\end{Example}

The proof of Example \ref{ex:ls} is given in Section \ref{sec:ls}.

\begin{Remark}
\label{rmk:nonlinear_kernel}
    Example \ref{eq:ls} includes the (nonlinear) least-squares estimators  
proposed in \cite{filippi2010parametric, 
hao2020adaptive} for (generalized) linear models 
as special cases. 
In these settings, 
   $\mu(\theta,x,a)=b(\psi(x,a)^\top\theta)$,
   where $\psi:  \cX\times \cA\to \sR^{ d}  $ is a given kernel function and $b:\sR \to \sR$ is a sufficiently regular link function.
    Then 
    $(\nabla_\theta \mu) (\theta, x,a)
    (\nabla_\theta \mu) (\theta, x,a)^\top = (b'(\psi(x,a)^\top\theta))^2
    \psi(x,a)\psi(x,a)^\top$
    and 
    $(\nabla^2_\theta \mu) (\theta, x,a)= b''(\psi(x,a)^\top\theta) \psi(x,a)\psi(x,a)^\top$.
    Hence the desired constants $c_1$ and $c_2$ exist if $\min_{x\in \sR} |b'(x)|>0$,
    $\max_{x\in \sR} |b'(x)|<\infty$,
    and
    $\max_{x\in \sR} |b''(x)| $ is sufficiently small, which holds in particular if $b$ is a non-constant  affine function.   The   information matrix $H$ is the kernel's covariance $\psi\psi^\top$ (scaled by an  appropriate positive constant) as observed in \cite{filippi2010parametric, 
hao2020adaptive}. 
\end{Remark}

Now we present the main theorem of this section, which quantifies the accuracy of the estimator ${\theta}_T$ defined in \eqref{eq:ls}
using the information matrix $H$ in 
  (H.\ref{assum:loss}).
\begin{Theorem}
\label{thm:statistical error}
Suppose (H.\ref{assum:pi_theta}) and (H.\ref{assum:loss}) hold.
For all $T\in \sN$ and all  $\delta >0$, 
$$\sP \left( \big\|{\theta_{T}} - {\theta^\star} \big\|^2_{\sR^d} \leq  
8 \rho_{\min} (V_T)^{-1} \left(\ln \big( {\det V_T} \big) + 2 \ln \left( \tfrac{1}{\delta}\right)+ \big\|\theta^\star \big\|_{V_T^{-1}}^2\right)
   \right) \geq 1 - \delta.$$
where $V_T \coloneqq \sum_{t=1}^T H(x_t, a_t) +  2 I_{d}$.
\end{Theorem}

The proof of Theorem \ref{thm:statistical error} is given in Section \ref{sec:proof_stat_error}.
Theorem \ref{thm:statistical error} indicates that    the accuracy  of 
${\theta}_T$ can be measured   by  the minimum eigenvalue of 
$V_T$.  
This suggests us to  design a   learning algorithm 
  such that  $\rho_{\min}(V_T)$  blows up to infinity at an appropriate rate
 as $T\to \infty$.

\subsection{$\epsilon$-policy gradient algorithm and its regret}

The section introduces an $\epsilon$-policy gradient algorithm that explores the environment with probability $\epsilon$ and exploits using a gradient descent update.
To this end, we assume that there exists an exploration policy associated to the loss \eqref{eq:ls}.

\begin{Assumption}
\label{assum:explore}
Let  $H : \cX \times \cA \to \sS^{d}_{\ge 0}$    be given in (H.\ref{assum:loss}).
There exists  a known Markov   kernel 
   $\pi_{\textsf{Exp}}:\cX\to   \cP(\cA)$
   and a known constant   $\rho_H >0$ 
   such that $  \int_{\cX \times \cA} H(x,a) \pi_{\textsf{Exp}}(\d a|x) \mu (d x)  \succeq      \rho_H I_d$.
   Moreover, there exists a known constant $C_H\ge 0$ such that 
   $\sup_{(x,a)\in \cX\times \cA}\|H(x,a)\|_{\op}\le C_H$.
\end{Assumption}

Assumption  (H.\ref{assum:explore}) ensures that the agent has access to   a policy $\pi_{\textsf{Exp}}$ such that observations generated by $\pi_{\textsf{Exp}}$ guarantees  to explore the parameter space   in expectation.
This condition is commonly imposed in the literature to facilitate learning 
 \cite{filippi2010parametric, agrawal2013thompson, chowdhury2017kernelized, vakili2021information,  janz2023exploration}.
 It typically holds where the action space is   sufficiently rich
for exploring the parameter space. We assume for simplicity that $H$ is bounded to quantify the  
behavior of $(x,a) \mapsto H(x,a)$ under the measure $\pi_{\textsf{Exp}}(\d a|x) \mu (d x)$ in high probability.

The following proposition    provides sufficient conditions for (H.\ref{assum:explore}) if $\pi_\theta$ is given by the generalized linear model in Example \ref{example:negative log loss}. It extends  the conditions in \cite{filippi2010parametric}
for bandits with finite action spaces   to the present setting with a general action space $\cA$  and an additional feature space $\cX$.
The proof is given in Section 
\ref{sec:log-likelihood}.

\begin{Proposition}
 \label{prop:glm_explore}
Let  $\pi_\theta$ be given as in Example \ref{example:negative log loss}. 
Assume   that 
$\cX$ and $\cA$ are  topological spaces with the associated Borel $\sigma$-algebras, 
and  $\mu\in \cP(\cX)$ has   support $\operatorname{supp}(\mu)$.\footnotemark{}
  Assume that
     $\psi : \operatorname{supp}(\mu) \times \cA \to \sR^{m \times d}$ is continuous and bounded,
  and
   $ \operatorname{span}( \{\psi(x,a )^\top \mid   x\in \operatorname{supp}(\mu), a\in \cA \})=\sR^d$.
  Then (H.\ref{assum:explore}) holds with $\pi_{\textsf{Exp}}(da|x)=\eta(da)$,
  for any measure  $\eta\in \cP(\cA)$  with $\operatorname{supp}(\eta)=\cA$.

\footnotetext{Given  a topological space $X$  equipped   with its Borel $\sigma$-algebra, 
the support $\operatorname{supp}(\mu)$ of a measure $\mu$, if it exists, 
is a closed set satisfying (1) $\mu(\operatorname{supp}(\mu)^c)=0$,  
and (2) if $G$ is open and $G\cap \operatorname{supp}(\mu)\not =\emptyset$,
then $\mu(G\cap \operatorname{supp}(\mu))>0$.}
  
\end{Proposition}

For the least-squares loss in Example \ref{ex:ls},
it is easy to see that 
if  the mean of the response variable 
  admits a nonlinear kernel representation as  discussed in Remark \ref{rmk:nonlinear_kernel},
  then 
  the same 
policy $\pi_{\textsf{Exp}}$
given in   Proposition \ref{prop:glm_explore}
is also an exploratory policy  for the    least-squares loss function \eqref{eq:loss_ls}.
This is because in this case, the corresponding information matrix $H$ can also be selected as    the   (scaled) kernel's covariance,
 akin to that for the  log-likelilhood loss function \eqref{eq:loss_glm}  in Example \ref{example:negative log loss}.
The construction of an exploration policy for general nonlinear models in Example \ref{ex:ls} is more technically involved and is left for future work.

We then introduce   structural properties of the expected reward
    for the design and analysis of the gradient descent updates.
Recall  that under (H.\ref{assum:pi_theta}),
the expected reward 
$\bar r_\theta(x,a) =\int_\cY r(x,a, y)\pi_\theta (\d y|x,a)$
is well-defined for all $(\theta, x,a)\in \sR^d\times \cX\times \cA$,
and $\theta^\star$ is the (unknown) true parameter     in (H.\ref{assum:pi_theta}).

\begin{Assumption}
\label{assum:reward}
$\cA$ is a Hilbert space equipped with the  norm  $ \|\cdot\|_{\cA}$.
For all $x\in \cX$, 
$ a\mapsto  \bar r_{\theta^\star} (x,a) $ is Fr\'{e}chet  differentiable, 
$\sup_{(x,a)\in \cX\times \cA} |\bar r_{\theta^\star} (x,a)|<\infty$, 
and there exist  
 $ L_a\ge \gamma_a> 0$ such that 
 for all $x\in \cX$ and $a,a'\in  \cA$, 
 $\|(\nabla_a \bar r_{\theta^\star})( x, a')-(\nabla_a \bar r_{\theta^\star})( x, a)\|_{\cA}\le L_a  \|a'-a\|_{\cA}$ and
 \begin{equation}
 \label{eq:pl}
\sup_{a\in \cA} \bar r_{\theta^\star} (x, a) -  \bar r_{\theta^\star} ( x, a)  \leq \frac{1}{2\gamma_a} \|(\nabla_a \bar r_{\theta^\star})( x, a) \|^2_{\cA},
 \end{equation}
 and   %there exists    $L_\theta \geq 0$ such that for all $(\theta, x, a) \in \sR^d\times  \cX\times \cA$,
 \begin{equation}
  \label{eq:stability}
L_\theta\coloneqq \sup_{\theta\not =\theta^\star, x\in \cX, a\in\cA} \frac{\| (\nabla_a \bar r_{\theta})( x, a) - (\nabla_a \bar r_{\theta^\star})( x, a)  \|_{\cA}}{\|\theta - \theta^\star \|_{\sR^d}}<\infty.
 \end{equation}
  The agent knows   the constant $L_a$. 
\end{Assumption}

Assumption (H.\ref{assum:reward}) assumes   for simplicity  that the action space is a Hilbert space, and 
the expected reward is Fr\'{e}chet differentiable in action. 
This allows for 
developing a policy gradient   update   based on the Fr\'{e}chet derivative of the  reward in the action. 
Similar analysis can be performed if $\cA$ is a convex subset of a Hilbert space. In that case,   the gradient ascent update can be replaced by a projected gradient ascent to incorporate the action constraints.
Condition \eqref{eq:pl} is commonly known as the Polyak-\L ojasiewicz condition \cite{karimi2016linear},
and guarantees linear convergence of gradient ascent  using \emph{exact gradients}.
 It is strictly weaker than the strong concavity  condition and 
 is  satisfied 
 by  many practically important nonconvex/nonconcave optimization problems,
 such  linear neural networks with  suitable initializations \cite{li2018algorithmic}, nonlinear neural networks in the so-called neural tangent kernel regime \cite{liu2022loss}. 
%
% reinforcement learning with linear policy paramterisation  \cite{fazel2018global,hambly2021policy, giegrich2024convergence}
% and softmax policy paramterisation \cite{mei2020global}. 
Condition \eqref{eq:stability} asserts the Lipschitz continuity   of the reward's gradient    with respect to   the response distribution,
which allows for   quantifying the regret in the exploitation step.

Now we are ready to  present the $\epsilon$-PG algorithm.

\begin{algorithm}[H]
\label{Alg: epsilon greedy}
\DontPrintSemicolon
\SetAlgoLined

%\KwResult{Write here the result }
  \KwInput{Exploration policy $\pi_{\textsf{Exp}}:\cX\to   \cP(\cA) $, 
 exploration rates  $(\epsilon_t)_{t \in \sN}  \subset [0,1]$, 
 initial policy  $\phi_0:\cX\to \cA$ and
 learning rates $(\eta_t)_{t \in \sN} \subset  [0,\infty)$.
  }

 \For{$t = 1, 2,\ldots$}
 {
 {Observe a feature $x_t$.}\;
 {Sample $\xi_t $ from   the uniform distribution on $(0,1)$.}\;
\eIf { $\xi_t < \epsilon_t$}{
            
            Execute $a_t$ sampled  from    
 $\pi_{\textsf{Exp}}(\cdot|x_t)$.
     }{
            Execute $a_t = \phi_{t-1}(x_t)$, 
        }
  {
Observe a  response $y_t$, and update  
 $\theta_t $ by \eqref{eq:ls}.
 }\;
 
  {
Update the policy     $ \phi_{t}=\phi_{t-1}+\eta_t (\nabla_a \bar r_{\theta_{t}})(\cdot,\phi_{t-1}(\cdot))$.
 }\;

}
 \caption{$\epsilon$-policy gradient ($\epsilon$-PG) algorithm}
\end{algorithm}

 The following theorem refines the estimate in   Theorem \ref{thm:statistical error},
 and 
 quantifies the   accuracy  of $(\theta_T)_{T\in \sN}$ in  Algorithm \ref{Alg: epsilon greedy}   in terms of the exploration rates $(\epsilon_t)_{t\in \sN}$.

\begin{Theorem}
    \label{Thm: convergence_epsilon}
    Suppose (H.\ref{assum:pi_theta}), (H.\ref{assum:loss})  and (H.\ref{assum:explore}) hold.
Consider Algorithm \ref{Alg: epsilon greedy}   with   exploration rates     $(\epsilon_t)_{t \in \sN}  \subset [0,1)$.
Then   for all $\delta \in (0,3/\pi^2] $, if 
\begin{equation}
M(\epsilon,T,\delta) = 2 +  \rho_H \sum_{t=1}^T \epsilon_t -  {2C_H} \left( \ln \frac{2dT^2}{\delta} + \sqrt{ 2 T  \ln \frac{2dT^2}{\delta}} \right)>0, \quad \forall  T\in \sN, 
\end{equation}
then  with probability at least $ 1-\frac{\pi^2}{3}\delta$ that
\begin{equation}
  \big\|{\theta_{T}} - {\theta^\star} \big\|^2_{\sR^d} \leq  
 \frac{  8 d \ln \left(  2 + C_H \left (  2 T + 6 \ln \frac{2dT^2}{\delta} \right)\right) + 16 \ln \left( \tfrac{2T^2}{\delta}\right)+ 2 \big\|\theta^\star \big\|_{\sR^d}^2}{M(\epsilon,T,\delta) },
 \quad \forall T\in \sN.
\end{equation} 
 
\end{Theorem}

The proof of Theorem \ref{Thm: convergence_epsilon} is given in Section \ref{sec:proof_error_epsilon}.

 The following theorem optimizes the learning rates $(\eta_t)_{t\in \sN}$ and  the exploration rates 
 $(\epsilon_t)_{t\in \sN}$ so   that  Algorithm \ref{Alg: epsilon greedy} 
achieves a regret of the order
$\cO(\sqrt T )$ (up to a logarithmic term)  in expectation.
This  regret bound  matches the lower bound for online convex optimization algorithms \cite{hazan2016introduction} and (instant-independent) regret bounds for parametric multi-armed bandits \cite{filippi2010parametric}.
 
\begin{Theorem}
\label{Thm: expected_regret}
Suppose   (H.\ref{assum:pi_theta}), (H.\ref{assum:loss}), (H.\ref{assum:explore}) and  (H.\ref{assum:reward}) hold.
If one sets   $\eta_T\equiv \eta\in (0, 1/L_a]$ for all $T\in \sN$, 
and 
$c_1 \sqrt{T} \ln (  d T )  \le \sum_{t=1}^T \epsilon_t\le c_2 \sqrt{T } \ln (  d T )$ 
for all $T\in \sN$, 
with constants $c_2\ge  c_1 \ge 30 {C_H}/{\rho_H}    $. 
Then there exists a constant $C\ge 0$ such that  
the regret of   Algorithm \ref{Alg: epsilon greedy}  satisfies 
 $\sE\left[ \operatorname{Reg}\left((a_t)_{t=1}^T\right)\right]  \le Cd \sqrt{T} \ln (  d T )  $ for all $T\ge 2$. 
\end{Theorem}

The proof of Theorem \ref{Thm: expected_regret} is given in Section \ref{sec:proof_expected_regret}.
The precise expression of the constant $C$ 
can be found in the proof. It  
depends only on the learning rate  $\eta$,
the constant $c_2$, 
 the sup-norm of the reward $r$, the Euclidean norm of $\theta^\star$,
 and 
 the   constants $L_\theta$, $C_H$ and $\gamma_a$  in  (H.\ref{assum:explore}) and  (H.\ref{assum:reward}).

\section{Proofs of main results}

\subsection{Proof of   Theorem \ref{thm:statistical error}}
\label{sec:proof_stat_error}

The following   lemma bounds the gradient of the loss $\ell$ 
using the information matrix $H$, which will be used in the proof of 
 Theorem \ref{thm:statistical error}.

\begin{Lemma} 
\label{lemma: self normalising}
Suppose (H.\ref{assum:pi_theta}) and (H.\ref{assum:loss}) hold.
For each $T\in \sN\cup\{0\}$, define the random variables 
$S_T \coloneqq \sum_{t=1}^T \nabla_{\theta}   \ell ({\theta}^\star, x_t, a_t, y_t  )   $
and $ V_T \coloneqq \sum_{t=1}^T H(x_t,a_t ) +  2 I_{d}$.
Then for all $T\in \sN$ and  $\delta >0$, 
$$\sP\left( \big\|S_T \big\|^2_{V_T^{-1}} \leq  \ln \left( {\det V_T}\right) + 2 \ln \left( \tfrac{1}{\delta}\right) \right) \geq 1 - \delta.$$
where 
$\|v\|^2_A \coloneqq v^\top A v$ for all $v\in \sR^d$ and $A\in \sS^d_{\ge 0}$. 
\end{Lemma}

\begin{proof}

The proof adapts the argument presented in \cite[Lemma 4.5]{szpruch2021exploration} to the current context, and we provide the details here for the reader's convenience.
We first rewrite  $\exp \left( \frac{1}{2}\|S_T\|^2_{V_T^{-1}} \right)$ as an  integral of some   super-martingales.
For for each $A \in \sS^{d}_+$, let $c(V) \coloneqq (2 \pi)^{d
/2} (\det A)^{-1/2} = \int_{\sR^{d}} \exp\left( - \|\lambda\|^2_A \right) \d \lambda$
 be the normalising constant corresponding to the density of the normal distribution $\cN(0, A^{-1})$. 
 By completing the square,  
\begin{align*}
&\exp\big(\tfrac{1}{2} \|S_T\|^2_{V_T^{-1}} \big) = \frac{1}{c(V_T)}\int_{\sR^{d}}  \exp \Big( \tfrac{1}{2} \|S_T\|^2_{V_T^{-1}} - \tfrac{1}{2} \big\|\lambda - V_T^{-1} S_T\big\|^2_{V_T}\Big) \d \lambda \\
&= \frac{1}{c(V_T)}\int_{\sR^{d}}  \exp \Big( \lambda^\top S_T   - \tfrac{1}{2}\| \lambda \|_{V_T}^2 \Big)  \,\d \lambda  =   \frac{1}{c(V_T)}\int_{\sR^{d}} \exp \Big( \lambda^\top S_T  - \tfrac{1}{2}\| \lambda \|_{V_T-I_{d}}^2  - \tfrac{1}{2}\| \lambda \|^2_{I_{d}}\Big)  \,\d \lambda \\
& =  
 \frac{1}{c(V_T)}\int_{\sR^{d}} M^\lambda_T \exp \big( - \tfrac{1}{2}\| \lambda\|^2_{I_{d}}\big)  \,\d \lambda 
 =  \left(\frac{\det V_T}{\det I_{d}} \right)^{1/2} \frac{ 1 }{c(I_{d})} \int_{\sR^{d}} M^\lambda_T  \exp \Big(-\tfrac{1}{2} \|\lambda\|^2_{I_{d}} \Big) \,\d \lambda,
 \end{align*}
where for each $\lambda \in \sR^{d}$, $M^\lambda_T$ is defined by
$$M^\lambda_T = \exp \left( \lambda^\top S_T - \lambda^\top (V_T - I_{d}) \lambda\right) = \exp\left( \sum_{t=1}^T \Big( 
\lambda^\top
\nabla_\theta \ell ( {\theta}^\star, x_t, a_t, y_t  )  - \lambda^\top H(x_t,a_t) \lambda \Big) -\|\lambda\|^2_{\sR^d}\right).$$
For each $T\in \sN\cup\{0\}$,
define the $\sigma$-algebra 
$\cF_T \coloneqq \sigma\{ (x_t, a_t, y_t)_{t=1}^n, x_{T+1}, a_{T+1})$.
Clearly 
for all $\lambda \in \sR^{d}$, 
$(M_t^\lambda)_{t=0}^\infty$ is adapted to the filtration $(\cF_t)_{t=0}^\infty$, 
and by    (H.\ref{assum:loss}\ref{item: compatibility}),
$\sE[M_{T+1} ^\lambda \mid \cF_T]\le M_{T}^\lambda$ for all $T\in \sN\cup\{0\}$, which implies that 
$\sE[M^\lambda_T] \leq \sE[M^\lambda_0] \le 1$ for all $T\in \sN\cup\{0\}$.

Therefore, by  Markov's inequality and  Fubini's theorem,
\begin{align*}
&\sP  \left( \big\|S_T \big\|^2_{V_T^{-1}} >  \ln \left( \tfrac{\det V_T}{\det I_{d}} \right) + 2 \ln \left( \tfrac{1}{\delta}\right) \right)  
= \sP \bigg(\exp \Big( \tfrac{1}{2 }\big\|S_T \big\|^2_{V_T^{-1}} \Big) >  \frac{1}{\delta} \left(\frac{\det V_T}{\det I_{d}} \right)^{1/2} \bigg)
\\ 
&\quad = \sP \left(\frac{ 1 }{c(I_{d})} \int_{\sR^{d}} M^\lambda_T  \exp \Big(-\tfrac{1}{2} \|\lambda\|^2_{I_{d}} \Big)
 \,\d \lambda >  \frac{1}{\delta} \right)
  \leq \delta \sE  \left[ \frac{ 1 }{c(I_{d})} \int_{\sR^{d}} M^\lambda_T  \exp \Big(-\tfrac{1}{2} \|\lambda\|^2_{I_{d}} \Big)
 \,\d \lambda \right] \\
 &\quad = \frac{ \delta}{c(I_{d})} \int_{\sR^{d}} \sE \big[M^\lambda_T \big]   \exp \Big(-\tfrac{1}{2} \|\lambda\|^2_{I_{d}} \Big)
 \,\d \lambda \leq   \frac{ \delta }{c(I_{d})} \int_{\sR^{d}}  \exp \Big(-\tfrac{1}{2} \|\lambda\|^2_{I_{d}} \Big) = \delta,
\end{align*}
 where the last inequality used from the fact that $\sE[M^\lambda_T] \leq 1$ for all $\lambda \in \sR^{d}$ and $T \in \sN$.
 This along with $\det I_d =1$ proves the desired estimate.
\end{proof}

\begin{proof}[Proof of Theorem \ref{thm:statistical error}]

Fix $T\in \sN$. 
For each $\tau\in [0,1]$, define  
the random variables $\theta^{ \tau}= \theta^\star+(1-\tau)(  {\theta}_T-\theta^\star) $ 
and 
$$F(\tau)  \coloneqq  ({\theta^\star} - {\theta_{T}})^\top 
\left( \sum_{t=1}^T \nabla_{\theta} \ell (\theta^{ \tau}, x_t, a_t,y_t ) +  \theta^{ \tau}\right).
$$
Note that by  \eqref{eq:ls},
$ \sum_{t=1}^T \nabla_{\theta} \ell (  {\theta}_T, x_t, a_t,y_t ) +   {\theta}_T=0$
and hence  $F(0)=0$. 
As $\ell$ is twice differentiable in $\theta$, by the mean value theorem, there exists $\tau \in (0,1)$ such that
 \begin{equation}
 \label{eq:F1_F'}
         F(1) = F(0) + F'(\tau) = F'(\tau).
\end{equation}

Define $S_T =  \sum_{t=1}^T \nabla_{\theta} \ell  (  \theta^\star, x_t, a_t,y_t ))$ 
as in Lemma \ref{lemma: self normalising}. 
Then 
\begin{align}
\label{eq:F1}
\begin{split}
         |F(1)|^2 &= ({\theta^\star} - {\theta_{T}})^\top (S_T+\theta^\star) (S_T+\theta^\star)^\top({\theta^\star} - {\theta_{T}}) 
         \\
         & = \tr\left(({\theta^\star} - {\theta_{T}})({\theta^\star} - {\theta_{T}})^\top (S_T+\theta^\star) (S_T+\theta^\star)^\top\right) \\
    &=     \tr \left( V_T^{1/2}({\theta^\star} - {\theta_{T}})({\theta^\star} - {\theta_{T}})^\top V_T^{1/2} V_T^{-1/2} 
(S_T+\theta^\star) (S_T+\theta^\star)^\top  V_T^{-1/2}\right) \\
 &\leq \tr\left( V_T^{1/2}({\theta^\star} - {\theta_{T}})({\theta^\star} - {\theta_{T}})^\top V_T^{1/2} \right) 
  \tr\left(V_T^{-1/2} (S_T+\theta^\star) (S_T+\theta^\star)^\top  V_T^{-1/2}\right) 
 \\
 & =\big\|V_T^{1/2} ({\theta^\star} - {\theta_{T}}) \big\|^2_{\sR^d} \big\|S_T+\theta^\star\big\|_{V_T^{-1}}^2,
 \end{split}
\end{align}
where the second to last inequality used  the fact that $\tr(AB) \leq \tr(A)\tr(B)$ for $A,B \in \sS^d_{\ge 0}$.

By the chain rule and (H.\ref{assum:loss}\ref{item: convex}),  
     \begin{align}
     \begin{split}
     \label{eq:F'}
         |F'(\tau)| &= ({\theta^\star} - {\theta_{T}})^\top 
 \left( \sum_{t=1}^T \nabla^2_{\theta} \ell \big(\theta^\tau, x_t,a_t,y_t \big) +   I_{d}  \right) ({\theta^\star} - {\theta_{T}}) \\
 &\geq  ({\theta^\star} - {\theta_{T}})^\top   (V_T - I_{d}) ({\theta^\star} - {\theta_{T}}) 
 \\
 &=    ({\theta^\star} - {\theta_{T}})^\top V_T^{ 1/2}  V_T^{-1/2}  (V_T - I_{d})V_T^{-1/2} V_T^{1/2}  ({\theta^\star} - {\theta_{T}}) 
 \\
 &\ge  \| V_T^{1/2} ({\theta^\star} - {\theta_{T}}) \|^2_{\sR^d} \rho_{\min} (V_T^{-1/2} (V_T - I_{d})  V_T^{-1/2}\big).
 \end{split}
     \end{align}
     Combining \eqref{eq:F1_F'}, \eqref{eq:F1} and \eqref{eq:F'} yields 
$$ \big\|V_T^{1/2} ({\theta^\star} - {\theta_{T}}) \big\|^2_{\sR^d}  \leq \rho_{\min} (V_T^{-1/2} (V_T - I_{d})  V_T^{-1/2}\big)^{-2}\big\|S_T+\theta^\star \big\|^2_{V_T^{-1}},$$
which implies that 
\begin{align}
\label{eq:theta_star_T}
 \big\| {\theta^\star} - {\theta_{T}} \big\|^2_{\sR^d}  \leq \rho_{\min }(V_T)^{-1}\rho_{\min} (V_T^{-1/2} (V_T - I_{d})  V_T^{-1/2}\big)^{-2}\big (\|S_T \big\|_{V_T^{-1}}+\|\theta^\star\|_{V_T^{-1}})^2.
 \end{align}

We claim that   $\rho_{\min} (V_T^{-1/2} (V_T - I_{d})  V_T^{-1/2}\big)\ge 1/2$. To see it, let $H_T=   \sum_{t=1}^T H(x_t, a_t)\in  {\sS}^d_{\ge 0}$. Then $V_T=H_T+2I_d$,
and 
\begin{align*}
 \rho_{\min} (V_T^{-1/2} (V_T - I_{d})  V_T^{-1/2}\big) 
 &=\min_{x\not =0}\frac{x^\top V_T^{-1/2} (V_T - I_{d})  V_T^{-1/2} x}{\|x\|^2_{\sR^d}}
 =\min_{y\not =0}\frac{y^\top (V_T - I_{d})  y}{\|V^{1/2}_T y\|^2_{\sR^d}}
 \\
 &=\min_{y\not =0}\frac{y^\top  H_T y+y^\top y }{y^\top  H_T y+2y^\top y }
 =1- \min_{y\not =0}\frac{ y^\top y }{y^\top  H_T y+2y^\top y }
 \ge \frac{1}{2},
\end{align*}
as $y^\top  H_T y \ge 0$. 
Hence by \eqref{eq:theta_star_T},
$$ \big\| {\theta^\star} - {\theta_{T}} \big\|^2_{\sR^d}  \leq 8\rho_{\min} (V_T)^{-1} (\big\|S_T \big\|_{V_T^{-1}}^2+ \big\|\theta^\star \big\|_{V_T^{-1}}^2),$$
which along with Lemma  \ref{lemma: self normalising} yields the desired   result. 
\end{proof}

\subsection{Proof of Theorem \ref{Thm: convergence_epsilon}}
\label{sec:proof_error_epsilon}

The following lemma establishes  
a   concentration inequality of  the information matrix $(H(x_t,a_t))_{t\in \sN}$,
which will be used in the proof of Theorem \ref{Thm: convergence_epsilon}.
The proof of the lemma is given in Section \ref{sec:proof_sub_exponential}.

\begin{Lemma}
\label{lemma:sub-exponential condition}
    Suppose  (H.\ref{assum:pi_theta}) and  (H.\ref{assum:explore}) hold.
    For all $n\in \sN\cup\{0\}$, let $\cF_n =\sigma\{(x_k,a_k,y_k)_{k=1}^n\}$ and 
    let  $H_n= H(x_n,a_n)$.
Then for all  $n\in \sN$ and $\delta \in (0,1]$,
$$\sP \left(\left\| \sum_{k=1}^n \Big( H_k - \sE[H_k |\cF_{k-1}] \Big)\right\|_{\op} \leq  \frac{2C_H}{3} \left(\ln \frac{d}{\delta} + \sqrt{ \left(\ln \frac{d}{\delta}\right)^2+18 n  \ln \frac{d}{\delta}} \right)\right) \geq 1-\delta.$$
 
\end{Lemma}

\begin{proof}[Proof of Theorem \ref{Thm: convergence_epsilon}]
Fix $\delta \in (0,1]$.  
For each $T\in \sN$, let 
$V_T \coloneqq \sum_{t=1}^T H(x_t, a_t) +  2I_{d}$.
  By Theorem  \ref{thm:statistical error} (with $\delta' = \delta/2$) and Lemma \ref{lemma:sub-exponential condition} (with $\delta' = \delta/2$), 
  for all $T\in \sN$, 
with   probability at least $1-\delta$,     
\begin{equation}
        \label{eq: convergence phase1}
  \big\|{\theta_{T}} - {\theta^\star} \big\|^2_{\sR^d} \leq  
8 \rho_{\min} (V_T)^{-1} \left(\ln \big( {\det V_T} \big) + 2 \ln \left( \tfrac{2}{\delta}\right)+ \big\|\theta^\star \big\|_{V_T^{-1}}^2\right),
\end{equation}   
and 
\begin{equation}
        \label{eq: convergence phase2}
\left\| \sum_{t=1}^T \Big( H(x_t, a_t)- \sE[H(x_t, a_t)|\cF_{t-1}] \Big)\right\|_{\op} \leq  \frac{2C_H}{3} \left(\ln \frac{2d}{\delta} + \sqrt{ \left(\ln \frac{d}{\delta}\right)^2+18 T  \ln \frac{2d}{\delta}} \right).
\end{equation}   
 where $\cF_t =\sigma\{(x_k ,a_k ,y_k)_{k=1}^t\}$.
In the sequel, we carry out the estimate conditioned on the above event.
Note by  (H.\ref{assum:explore}) and the  exploration step  in Algorithm \ref{Alg: epsilon greedy}, 
 \begin{equation}
 \label{eq: convergence phase4}
     \rho_H \sum_{t=1}^T \epsilon_t   I_{d} \preceq  \sum_{t=1}^T   \sE[H(x_t, a_t)|\cF_{t-1}] \preceq T C_H I_{d}.
 \end{equation}
Combining \eqref{eq: convergence phase2} with the upper bound in \eqref{eq: convergence phase4} yields that 
$$\left\| V_T \right\|_{\op} \leq 2 + TC_H +  \frac{2C_H}{3} \left(\ln \frac{2d}{\delta} + \sqrt{ \left(\ln \frac{d}{\delta}\right)^2+18 T  \ln \frac{2d}{\delta}} \right).$$
This along with the fact that  $\det A \leq \|A\|_{\op}^{d}$ for all    $A \in \sS^{d}_{\ge 0}$ yields 
\begin{equation}
    \label{eq: convergence phase5}
\begin{aligned}
\begin{split}
    &\ln \big( {\det V_T} \big) 
    \\
    &\leq d \ln \left(  2 + TC_H +  \frac{2C_H}{3} \left(\ln \frac{2d}{\delta} + \sqrt{ \left(\ln \frac{d}{\delta}\right)^2+18 T  \ln \frac{2d}{\delta}} \right)\right)
    \\
    &\le  d \ln \left(  2 + C_H \left (  T + \frac{4}{3} \ln \frac{2d}{\delta} +2\sqrt{2T\ln \frac{2d}{\delta}}\right)\right)
    \le    d \ln \left(  2 + C_H \left (  2 T + 6 \ln \frac{2d}{\delta} \right)\right).
    \end{split}
\end{aligned}
\end{equation}
Similarly, combining \eqref{eq: convergence phase2} with the lower bound in \eqref{eq: convergence phase4} yields
\begin{equation}
    \label{eq: convergence phase6}
\begin{aligned}
\begin{split}
  \rho_{\min} (V_T) &\geq 2 +  \rho_H \sum_{t=1}^T \epsilon_t - \frac{2C_H}{3} \left(\ln \frac{2d}{\delta} + \sqrt{ \left(\ln \frac{d}{\delta}\right)^2+18 T  \ln \frac{2d}{\delta}} \right)
  \\
  &\ge 2 +  \rho_H \sum_{t=1}^T \epsilon_t -  {2C_H} \left( \ln \frac{2d}{\delta} + \sqrt{ 2 T  \ln \frac{2d}{\delta}} \right).
  \end{split}
\end{aligned}
\end{equation}
Combining 
\eqref{eq: convergence phase1}, 
\eqref{eq: convergence phase5} and \eqref{eq: convergence phase6}  implies that for all   $\delta\in (0,1)$ and  $T\in \sN$, 
if 
\begin{equation}
\label{eq:condition_T}
\tilde{M}(\epsilon,T,\delta') = 2 +  \rho_H \sum_{t=1}^T \epsilon_t -  {2C_H} \left( \ln \frac{2d}{\delta'} + \sqrt{ 2 T  \ln \frac{2d}{\delta'}} \right)>0,
\end{equation}
then 
with probability at least $1-\delta$, 
\begin{equation}
\label{eq:estimate_T}
  \big\|{\theta_{T}} - {\theta^\star} \big\|^2_{\sR^d} \leq  
 \frac{  8 d \ln \left(  2 + C_H \left (  2 T + 6 \ln \frac{2d}{\delta'} \right)\right) + 16 \ln \left( \tfrac{2}{\delta'}\right)+ 2 \big\|\theta^\star \big\|_{\sR^d}^2}{\tilde M(\epsilon,T,\delta') }.
\end{equation}   
Now suppose that 
\begin{equation}
M(\epsilon,T,\delta) = 2 +  \rho_H \sum_{t=1}^T \epsilon_t -  {2C_H} \left( \ln \frac{2dT^2}{\delta} + \sqrt{ 2 T  \ln \frac{2dT^2}{\delta}} \right)>0, \quad \forall  T\in \sN, 
\end{equation}
then applying \eqref{eq:condition_T} with $\delta'=\delta/T^2$, 
it holds with probability at least $1-\delta\sum_{t=1}^\infty \frac{1}{T^2}=1-\frac{\pi^2}{3}\delta$ that for all $T\in \sN$,  
\begin{equation}
  \big\|{\theta_{T}} - {\theta^\star} \big\|^2_{\sR^d} \leq  
 \frac{  8 d \ln \left(  2 + C_H \left (  2 T + 6 \ln \frac{2dT^2}{\delta} \right)\right) + 16 \ln \left( \tfrac{2T^2}{\delta}\right)+ 2 \big\|\theta^\star \big\|_{\sR^d}^2}{M(\epsilon,T,\delta) }.
\end{equation} 
This proves the desired estimate. 
\end{proof}

\subsection{Proof of Theorem \ref{Thm: expected_regret}}
\label{sec:proof_expected_regret}

We start by 
quantifying  how errors in gradient evaluation propagate through the policy gradient iterates.
The proof of Lemma \ref{lemma:regret to action} is given in Section \ref{sec:proof_regret_action}.

\begin{Lemma}
\label{lemma:regret to action}
Suppose   (H.\ref{assum:reward}) holds.
Consider Algorithm \ref{Alg: epsilon greedy}   with   learning rates     $\eta_t\equiv \eta\in (0, 1/L_a]$ for all $t\in \sN$.
Then for all $T \in \sN $,
\begin{align*}
   & \sup_{a\in \cA }\bar r_{\theta^\star}(x,a )- \bar r_{\theta^\star}(x,\phi_{T}(x))
   \\
    & \leq 2  
   \left( \sup_{(x,a)\in\cX\times  \cA } |\bar r_{\theta^\star}(x,a )|\right)  \left(1-\gamma_a \eta \right)^{T} + \frac{\eta L_\theta^2}{2} \sum_{t= 1}^{T} (1-\gamma_a \eta)^{T-t} \|
  \theta_t-\theta^\star  \|^2_{\sR^d}.
\end{align*}
%where $\cE_t (x)= ( \nabla_a     r_{\theta^\star})( x,\phi_{t-1}(x)) - (\nabla_a r_{\theta_{t}})(x,\phi_{t-1}(x)) $.
\end{Lemma}

\begin{proof}[Proof of Theorem \ref{Thm: expected_regret}]
Throughout this proof, let 
  $C\ge 0$ be an absolute  constant, which  % is  independent of $T$ 
 may take a different value at each occurrence.
Note that   if
\begin{equation}
  M(\epsilon,T)  \coloneqq  2 +  \rho_H \sum_{t=1}^T \epsilon_t -  {2C_H} \left( \ln \frac{2\pi^2 dT^3}{3} + \sqrt{ 2 T  \ln \frac{2\pi^2 dT^3}{3}} \right)>0, 
\quad   \forall T\in \sN,
\end{equation}
 then
 for any given $T\in \sN$, 
  by Theorem \ref{Thm: convergence_epsilon} with $\delta=3/(\pi^2T)\in (0, 3/\pi^2]$,  
 there exists an   event $A_T$ such that   
 $\sP(A_T)\ge 1-1/T$ and 
\begin{equation*}
 \big\|{\theta_{T}} - {\theta^\star} \big\|^2_{\sR^d} \leq  
 \frac{  8 d \ln \left(  2 + C_H \left (  2 T + 6 \ln \frac{2\pi^2 dT^3}{3} \right)\right) + 16 \ln \left( \tfrac{2\pi^2 T^3}{3}\right)+ 2 \big\|\theta^\star \big\|_{\sR^d}^2}{M(\epsilon,T) },
 \quad \forall T\in \sN.
\end{equation*} 
Suppose that   for all $T\in \sN$,  $\sum_{t=1}^T \epsilon_t\ge C_1 \sqrt{T} \ln (2\pi^2 dT^3 /3) $ with 
$C_1\ge 5  {C_H}/{\rho_H}$, which 
  can be ensured by setting
 $\sum_{t=1}^T \epsilon_t\ge 30 {C_H}/{\rho_H} \sqrt{T} \ln (dT) $ for all $T\ge 1$.
Then for all $T\ge 1$, 
$$
M(\epsilon, T)\ge 2+ C_H   \sqrt{T} \ln \left(\frac{2\pi^2 dT^3}{3}\right)\left(5-\frac{2}{\sqrt T}-\frac{2 \sqrt{2}}{\sqrt{ \ln (2\pi^2 dT^3 /3)}}\right)
\ge C C_H   \sqrt{T} \ln ( 2 dT),
$$ 
and hence on the event $A_T$,
\begin{align}
\label{eq:theta_T_error_bound1}
\begin{split}
  \big\|{\theta_{T}} - {\theta^\star} \big\|^2_{\sR^d}
  & \leq  
C  \frac{    d \ln \left(  2 + C_H \left (  2 T + 6 \ln  {2d T}  \right)\right) +   \ln \left(  {2 T} \right)+   \big\|\theta^\star \big\|_{\sR^d}^2}{ C_H   \sqrt{T} \ln ( 2dT) }
\\
&
\le C 
\frac{ d+\ln C_H+ \big\|\theta^\star \big\|_{\sR^d}^2}{C_H \sqrt{T}}, \quad \forall T\in \sN.
\end{split}
\end{align} 
Now let $\Delta_t =\sup_{a\in \cA} \bar r_{\theta^\star} (x_t, a) - \bar r_{\theta^\star}(  x_t, a_t ) $ for all $t\in \sN$,
and let $C_r = \sup_{(x,a)\in \cX\times \cA} |\bar r_{\theta^\star} (x,a)|<\infty$ (see     (H.\ref{assum:reward})). 
By the definitions of the regret   \eqref{eq:regret}  and Algorithm    \ref{Alg: epsilon greedy}, 
for all $T\ge 1$,
\begin{align}
\label{eq:regret_expect_term1}
\begin{split}
  & \sE\left[ \textrm{Reg}\left((a_t)_{t=1}^T\right)\right] =\sum_{t=1}^T \sE[\Delta_t  ] 
  \\
& \le   \sum_{t=1}^T \sE[\Delta_t 1_{\xi_t<\epsilon_t }] + \sum_{t=1}^T \sE[\Delta_t 1_{\xi_t\ge \epsilon_t }1_{A^c_t}  ]+\sum_{t=1}^T \sE[\Delta_t 1_{\xi_t\ge \epsilon_t }1_{A_t}  ]
\\
&
   \le 
   \sum_{t=1}^T 2 C_r \sP(\xi_t<\epsilon_t)  +  \sum_{t=1}^T 2 C_r \sP(A^c_t) +  \sum_{t=1}^T \sE[\Delta_t 1_{\xi_t\ge \epsilon_t }  1_{A_t}  ]
   \\
 &\le  2 C_r  \sum_{t=1}^T   \epsilon_t   +2 C_r   \sum_{t=1}^T  \frac{1}{t}  +\sum_{t=1}^T \sE[\Delta_t  1_{A_t} 1_{\xi_t\ge \epsilon_t } ].
 \end{split}
 \end{align}
 For all $t\ge 1$,
on the event  $A_{t}$ and $\{\xi_t\ge \epsilon_t \}$,  
by Lemma \ref{lemma:regret to action} and \eqref{eq:theta_T_error_bound1}, 
\begin{align*}
\Delta_t   1_{A_{t-1}} 1_{\xi_t\ge \epsilon_t }
& =\sup_{a\in \cA} \bar r_{\theta^\star} (x_t, a) - \bar r_{\theta^\star}(  x_t, \phi_{t-1}(x_t) )1_{A_{t}}  
\\
  & \leq 
 \left(1-\gamma_a \eta \right)^{t-1} \left(
2C_r+ \frac{\eta L_\theta^2}{2} \sum_{s=1}^{t-1} (1-\gamma_a \eta)^{-s} \|
  \theta_{s}-\theta^\star  \|^2_{\sR^d}\right)1_{A_{t}}  
  \\
  & \leq 
 \left(1-\gamma_a \eta \right)^{t-1} \left(
2C_r+ \frac{\eta L_\theta^2}{2} \sum_{s=1}^{t-1} (1-\gamma_a \eta)^{-s} C\frac{\bar{C}}{\sqrt{s}} \right),
\end{align*}
where $\bar {C} =  { (d+\ln C_H+ \big\|\theta^\star \big\|_{\sR^d}^2)}/{C_H }$. 
This along with \eqref{eq:regret_expect_term1} shows that for all $T\ge 2$, 
\begin{align}
\label{eq:regret_expect_term2}
\begin{split}
  & \sE\left[ \textrm{Reg}\left((a_t)_{t=1}^T\right)\right]  
\\
&\le 
C C_r c_2\left(\sqrt{T} \ln (d T)+\ln T+ \frac{1}{\gamma_a \eta}  \right) 
+C \eta L_\theta^2 \bar{C} \sum_{t=1}^T \sum_{s=1}^{t-1} (1-\gamma_a \eta)^{t-1-s}s^{-1/2}
\\
&= 
C C_r c_2 \left(\sqrt{T} \ln (d T)+ \ln T+ \frac{1}{\gamma_a \eta}  \right) 
+C \eta L_\theta^2 \bar{C} \sum_{s=1}^{T-1} \sum_{t=s+1}^{T} (1-\gamma_a \eta)^{t-1-s}s^{-1/2}
\\
&\le 
C C_r c_2  \left(\sqrt{T} \ln (d T)+\ln T+  \frac{1}{\gamma_a \eta}  \right) 
+C \frac{ L_\theta^2 \bar{C}}{\gamma_a} \sum_{s=1}^{T-1}  s^{-1/2}
\\
&
\le Cd \max\left( \frac{C_r c_2}{\gamma_a \eta} ,  \frac{ L_\theta^2  (1+\ln C_H+ \big\|\theta^\star \big\|_{\sR^d}^2)}{C_H \gamma_a}    \right)\sqrt{T} \ln (d T).
 \end{split}
 \end{align}
This proves the desired regret bound.
\end{proof}

 \section{Proofs of technical results}
 \subsection{Proofs of   Example \ref{example:negative log loss} and  Proposition \ref{prop:glm_explore}}
\label{sec:log-likelihood}

\begin{proof}[Proof of  Example \ref{example:negative log loss}]
A straightforward computation using the definition of $\ell$ and $H$ shows that 
$\nabla^2_\theta \ell(\theta, x,a, y )= \frac{2c_1}{c_2}  \psi(x,a)^\top \nabla^2_w \tilde b(x,\psi(x,a))  \psi(x,a) \succeq  H(x,a) $.
This verifies  (H.\ref{assum:loss}\ref{item: convex}).

To prove  (H.\ref{assum:loss}\ref{item: compatibility}), recall that 
since for all   $( x,a ) \in    \cX \times \cA$,
$\psi(x,a)\theta^\star \in \cH$ and $\tilde{b}(x,w)=b(x,w)$ for all $w\in \cH$.
Hence for all $(x,a,y)\in \cX\times \cA\times \cY$,  
\begin{align*}
   \nabla_\theta  \ell (\theta^\star, x,a,y ) &= \frac{2c_1}{c_2} \left(- \psi (x,a)^\top h(x,y)   +  \psi(x,a)^\top \nabla_ w \tilde b (x, \psi(x,a) {\theta}^\star )\right) 
   \\
   &= \frac{2c_1}{c_2} \left(- \psi (x,a)^\top h(x,y)   +  \psi(x,a)^\top \nabla_ w   b (x, \psi(x,a) {\theta}^\star )\right). 
  \end{align*} 
    Hence  
 for all $ {\lambda} \in \sR^{ d}$, 
\begin{align*}
& \int_\cY \exp \left( \lambda^\top \nabla_\theta  \ell (\theta^\star, x,a, y ) \right) \pi_{\theta^\star } (\d y | x,a)
\\
&= \int_\cY g(x,y) \exp \left( 
  h(x,y)^\top \psi (x,a)(\theta^\star -\lambda \frac{2c_1}{c_2})
 \right) \overline{\nu}   (\d y)
 \\
 &\quad \times \exp\left(- b\big(x,  \psi(x,a) {\theta}^\star \big)
     +\frac{2c_1}{c_2}   \nabla_ w b \big(x, \psi(x,a) {\theta}^\star \big)^\top  \psi(x,a) \lambda \right)
 \\
 &=    \exp\left(b(x,\psi (x,a)(\theta^\star -\lambda \frac{2c_1}{c_2})) - b\big(x,  \psi(x,a) {\theta}^\star \big)
     +\frac{2c_1}{c_2}    \nabla_ w b \big(x, \psi(x,a) {\theta}^\star \big)^\top \psi(x,a)\lambda  \right),
\end{align*}
where the last inequality used the condition $  \int_{\cY}  g(x,y) \exp\left( h(x,y)^\top w\right) \overline{\nu} (\d y) = \exp\left(  b\big(x, w\big)  \right)   $ for all $w\in \sR^m$. 
Since $\nabla^2_w b(x,w) \preceq c_2 I_m$, by Taylor's expansion,  
\begin{align*}
& \int_\cY \exp \left( \lambda^\top \nabla_\theta  \ell (\theta^\star, x,a, y ) \right) \pi_{\theta^\star } (\d y | x,a)
\\
 &\le   \exp\left( 
     \frac{c_1^2}{c_2^2} \sup_{w\in \sR^m}  \lambda^\top \psi(x,a)^\top  \nabla^2_ w b \big(x, w \big)^\top \psi(x,a)\lambda  \right)
     \le \exp\left( 
       \lambda^\top   H (x,a) \lambda  \right).
\end{align*}
This finishes the proof.
\end{proof}

\begin{proof}[Proof of Proposition \ref{prop:glm_explore}]
By  the boundedness of $\psi$ an the definition of $H$ in  Example \ref{example:negative log loss},
it is clear that $  \int_{\cX \times \cA} \exp(r_H \|H(x,a)\|_{\sR^{d\times d}}) \eta(\d a ) \mu( d x) \leq 2$
for some $r_H>0$.
For all $v\in \sR^d\setminus \{0\}$,
\begin{align}
\label{eq:H_min}
v^\top \left( \int_{\cX \times \cA} H(x,a) \eta  (\d a) \mu (d x)\right) v = 
  \int_{\cX \times \cA} \|\psi(x,a) v\|_{\sR^d}^2 \eta  (\d a) \mu (d x) 
>0.
\end{align}
Suppose that the above inequality does not hold. Then 
$\psi(x,a)^\top v =0$ for $\eta\otimes \nu$-a.s.~$(x,a)\in \cX\times \cA$. 
By the continuity of $\psi$ and the definition of the support $\operatorname{supp}(\mu)$, 
$\psi(x,a) v =0$
for all $(x,a)\in   \operatorname{supp}(\mu) \times \cA $.
This along with 
  $\operatorname{span}\{\psi(x,a )^\top\mid x\in \operatorname{supp}(\mu), a\in \cA \}=\sR^d$
  implies that $v=0$ and hence leads to a contraction. 
  Recall that for $A\in \sS^d$,
  $\rho_{\min}(A)=\min_{v\in \sR^d, \|v\|_{\sR^d}=1}v^\top Av$.
The desired conclusion then follows from  the compactness of $\{v\in \sR^d\mid \|v\|_{\sR^d}=1\}$ 
   and \eqref{eq:H_min}.
   \end{proof}

\subsection{Proof of Example \ref{ex:ls}}
\label{sec:ls}

\begin{proof}[Proof of Example \ref{ex:ls}]
Note that 
$\ell =C_1\tilde \ell$ with 
$\tilde \ell   (\theta, x,a,y ) =\frac{1}{2}(y-\mu(\theta, x,a))^2$ being 
the unnormalized quadratic loss.
For all $(\theta, x,a ,y)\in \sR^d\times \cX\times \cA\in \cY$,
\begin{align}
    \nabla_\theta \tilde \ell   (\theta, x,a,y ) &=(y-\mu(\theta, x,a)) (\nabla_\theta \mu) (\theta, x,a),
    \label{eq:gradient_ls}
    \\
    \nabla^2_\theta \tilde \ell   (\theta, x,a,y ) &=(y-\mu(\theta, x,a)) (\nabla^2_\theta \mu) (\theta, x,a)
    +(\nabla_\theta \mu) (\theta, x,a)(\nabla_\theta \mu) (\theta, x,a)^\top.
        \label{eq:hessian_ls}
\end{align}

To verify (H.\ref{assum:loss}\ref{item: convex}),
by \eqref{eq:hessian_ls} and the condition of $\sH$,
\begin{align*}
     \nabla^2_\theta \tilde \ell   (\theta, x,a,y ) &\succeq -| y-\mu(\theta, x,a)| (\nabla^2_\theta \mu) (\theta, x,a)
    +(\nabla_\theta \mu) (\theta, x,a)(\nabla_\theta \mu) (\theta, x,a)^\top
    \\
    &
    \succeq 
  -| y-\mu(\theta, x,a)|
    \sH(x,a)+
    \frac{1}{c_1} \sH(x,a)
    \succeq  \left(\frac{1}{c_1}+\underline{y}  - \overline{y}\right)\sH(x,a).
\end{align*}
Multiplying both sides of the inequality by $C_1\ge 0$ yields
$$
\nabla^2_\theta   \ell   (\theta, x,a,y )  \succeq C_1 \left(\frac{1}{c_1}+\underline{y}  - \overline{y}\right)\sH(x,a)=H(x,a).
$$

To verify  (H.\ref{assum:loss}\ref{item: compatibility}),
let $\lambda\in \sR^d$ be fixed. 
Note that 
as 
$\cY\subset [\underline{y},\overline{y}] $,
by Hoeffding's lemma,
$$  \int_\cY \exp \left( \tilde{\lambda} \big(y - \mu(\theta^\star, x, a ) \big) \right) \pi_{{\theta^\star}} (\d y | x,a) \leq \exp \Big( \frac{ (\overline{y}- \underline{y})^2}{8}\tilde{\lambda}^2  \Big), \quad \forall   \tilde{\lambda}\in \sR.$$
Setting $\tilde{\lambda} =C_1\lambda^\top  (\nabla_\theta \mu) (\theta^\star, x,a)$ in the above inequality 
and use the condition of $\sH$
yield 
\begin{align*}
   & \int_\cY \exp \left( {\lambda}^\top   \nabla_\theta  \ell (\theta^\star, x,a, y )  \right) \pi_{{\theta^\star}} (\d y | x,a) 
   \\
   &\quad 
   \leq \exp \left( \frac{ (\overline{y}- \underline{y})^2}{8}C_1^2 \lambda^\top  
    (\nabla_\theta \mu) (\theta^\star, x,a) 
    (\nabla_\theta \mu) (\theta^\star, x,a)^\top 
    \lambda    \right) 
    \\
    &\quad 
   \leq \exp \left( \frac{ (\overline{y}- \underline{y})^2}{8}C_1^2 \lambda^\top  
   c_2\sH(x,a) 
    \lambda    \right) 
    =   \exp \left(   \lambda^\top  
     H(x,a) 
    \lambda    \right), 
\end{align*}
where the last identity used 
$\frac{ (\overline{y}- \underline{y})^2}{8}C_1^2 c_2=C_1 \left(\frac{1}{c_1}+\underline{y}  - \overline{y}\right)$
due to the definition of $C_1$.
\end{proof}

\subsection{Proof of Lemma \ref{lemma:sub-exponential condition}}
\label{sec:proof_sub_exponential}

 Lemma \ref{lemma:sub-exponential condition} follows from  the following   concentration inequality
for  matrix-valued  bounded martingale, 
which has been established    in \cite[Theorem 1.2]{tropp2011freedman}.
 
\begin{Lemma}
\label{lemma:concentration_bounded_martingale}
Let $(\Omega, \cF, (\cF_n)_{n\in \sN\cup\{0\}}, \sP)$ be a filtered probability space,
let $(Y_n)_{n\in \sN\cup\{0\}}$ be an  $\sS^d_{\ge 0}$-valued martingale with respect to the filtration 
$(\cF_n)_{n\in \sN\cup\{0\}}$,
let $(X_n)_{n\in \sN}$ be the difference sequence such that $X_n =Y_n-Y_{n-1}$.  
Assume that there exists $R\ge 0$ such that 
$\|X_n\|_{\op}\le R$ for all $n\in \sN$. 
Define the  predictable quadratic variation process 
$
W_n =\sum_{k=1}^n \sE[X_kX^\top_k\mid \cF_{k-1}]
$ for all $n\in \sN$. 
Then for all $t\ge 0$ and $\sigma^2>0$,
$$
\sP\left(\exists n\in \sN\, \big\vert\, \|Y_n\|_{\op}\ge t, \|W_n\|_{\op}\le \sigma^2 \right)
\le d \exp\left(-\frac{t^2/2}{\sigma^2+Rt/3}\right).
$$
\end{Lemma}
 
 \begin{proof}[Proof of Lemma \ref{lemma:sub-exponential condition}]
For each $n\in \sN$, let 
$X_n = H_n - \sE[H_n|\cF_{n-1}]$ and define $Y_n =\sum_{k=1}^n X_k$. 
Note that $\|X_n\|_{\op}\le \|H_n\|_{\op} + \sE[\|H_n\|_{\op}|\cF_{n-1}] \le 2C_H$ due to (H.\ref{assum:explore}).
Fix  $n\in \sN$. By Lemma \ref{lemma:concentration_bounded_martingale} (with $R=2C_H$),
  for all $t\ge 0$ and $\sigma^2>0$,
\begin{equation}
\label{eq:Y_W_H}
\sP\left(  \|Y_n\|_{\op}\ge t, \|W_n\|_{\op}\le \sigma^2 \right)
\le d \exp\left(-\frac{t^2/2}{\sigma^2+2C_H t/3}\right),
\end{equation}
where $
W_n =\sum_{k=1}^n \sE[X_kX^\top_k\mid \cF_{k-1}]
$. Observe that for $\sP$-a.s.,
$$
\|W_n\|_\op \le \sum_{k=1}^n  \sE[\|X_kX^\top_k\|_\op\mid \cF_{k-1}] 
=  \sum_{k=1}^n  \sE[\|X_k \|^2_\op\mid \cF_{k-1}] \le n 4C_H^2,
$$
and hence by setting $\sigma^2=n4C_H^2$ in \eqref{eq:Y_W_H},  for all $t\ge 0$, 
$$
\sP\left(  \|Y_n\|_{\op}\ge t  \right)
=
\sP\left(  \|Y_n\|_{\op}\ge t, \|W_n\|_{\op}\le 4 n C_H^2 \right)
\le d \exp\left(-\frac{t^2/2}{4n C_H^2+2C_H t/3}\right).
$$
Hence for all $\delta\in (0,1)$, by choosing 
$$t= \frac{2C_H}{3} \ln \frac{d}{\delta} +\sqrt{\left(\ln \frac{d}{\delta}\frac{2C_H}{3}\right)^2+8 nC_H^2 \ln \frac{d}{\delta}} 
= \frac{2C_H}{3} \left(\ln \frac{d}{\delta} + \sqrt{ \left(\ln \frac{d}{\delta}\right)^2+18 n  \ln \frac{d}{\delta}} \right)
\ge 0,
$$ 
it holds that  $d \exp\left(-\frac{t^2/2}{4C_H^2+2C_H t/3}\right)= \delta$, 
which completes the desired estimate for $\delta\in (0,1)$. The conclusion for $\delta =1$ holds trivially. 
\end{proof}

\subsection{Proof of Lemma \ref{lemma:regret to action}}
\label{sec:proof_regret_action}

\begin{proof}[Proof of Lemma \ref{lemma:regret to action}]
Fix $x\in \cX$.
Observe that for all $t\in \sN $, by \eqref{eq:pl},  
\begin{equation}
\label{eq:pl_t}
2\gamma_a \left(\sup_{a\in \cA} \bar r_{\theta^\star} (x, a) -  \bar r_{\theta^\star} ( x, \phi_{t}(x) ) \right) \leq   \|(\nabla_a \bar r_{\theta^\star})( x, \phi_{t}(x) ) \|^2_{\cA}. 
 \end{equation}
For all $t\in \sN$, By the $L_a$-Lipschitz continuity of $a\mapsto \bar r_{\theta^\star}(x,a)$ and 
$ \phi_{t}=\phi_{t-1}+\eta (\nabla_a \bar r_{\theta_{t}})(\cdot,\phi_{t-1}(\cdot))$
with $\eta\le 1/L_a$, 
\begin{align}
\label{eq:lipschitz_smooth}
\begin{split}
   & \bar r_{\theta^\star}(  x, \phi_{t}(x)) 
   \\
    &\ge     \bar r_{\theta^\star}( x, \phi_{t-1}(x)) +
    \langle
    ( \nabla_a     \bar r_{\theta^\star})( x,\phi_{t-1}(x)), \phi_{t}(x)-\phi_{t-1}(x) \rangle_\cA - \frac{L_a}{2} \|\phi_{t}(x)-\phi_{t-1}(x)\|^2_{\cA}
     \\
    &=\bar r_{\theta^\star}( x, \phi_{t-1}(x)) +
    \langle
    ( \nabla_a     \bar r_{\theta^\star})( x,\phi_{t-1}(x)), \eta (\nabla_a \bar r_{\theta_{t}})(x,\phi_{t-1}(x))\rangle_\cA - \frac{L_a\eta^2}{2} \|  (\nabla_a \bar r_{\theta_{t}})(x,\phi_{t-1}(x))\|^2_{\cA}
    \\
    &\ge \bar r_{\theta^\star}( x, \phi_{t-1}(x)) +
    \langle
    ( \nabla_a     \bar r_{\theta^\star})( x,\phi_{t-1}(x)), \eta (\nabla_a \bar r_{\theta_{t}})(x,\phi_{t-1}(x))\rangle_\cA - \frac{\eta}{2} \|  (\nabla_a \bar r_{\theta_{t}})(x,\phi_{t-1}(x))\|^2_{\cA}
    \\
    &=
    \bar r_{\theta^\star}( x, \phi_{t-1}(x)) +\eta 
    \langle
    ( \nabla_a     \bar r_{\theta^\star})( x,\phi_{t-1}(x)),   (\nabla_a \bar r_{\theta_{t}})(x,\phi_{t-1}(x))\rangle_\cA - \frac{\eta}{2} \|  (\nabla_a \bar r_{\theta_{t}})(x,\phi_{t-1}(x))\|^2_{\cA}
    \\
    &=  \bar r_{\theta^\star}( x, \phi_{t-1}(x))  + \frac{\eta}{2} \|       ( \nabla_a     \bar r_{\theta^\star})( x,\phi_{t-1}(x))  \|^2_\cA - \frac{\eta}{2} \|
  \cE_t (x)  \|^2_\cA,
 \end{split}
\end{align}
with   $\cE_t (x)\coloneqq ( \nabla_a     \bar r_{\theta^\star})( x,\phi_{t-1}(x)) - (\nabla_a \bar r_{\theta_{t}})(x,\phi_{t-1}(x)) $.
Combining 
\eqref{eq:pl_t} and \eqref{eq:lipschitz_smooth} yields that for all $t\in \sN$, 
\begin{align}
\label{eq:recursive}
\begin{split}
&\sup_{a\in \cA} \bar r_{\theta^\star} (x, a) - \bar r_{\theta^\star}(  x, \phi_{t}(x)) 
\\
&\le 
\sup_{a\in \cA} \bar r_{\theta^\star} (x, a) - \bar r_{\theta^\star}(  x, \phi_{t-1}(x)) 
- \frac{\eta}{2} \|       ( \nabla_a     \bar r_{\theta^\star})( x,\phi_{t-1}(x))  \|^2_\cA + \frac{\eta}{2} \|
  \cE_t (x)  \|^2_\cA
\\
&\le 
(1-\gamma_a \eta) \left(\sup_{a\in \cA} \bar r_{\theta^\star} (x, a) - \bar r_{\theta^\star}(  x, \phi_{t-1}(x)) \right)
+ \frac{\eta}{2} L_\theta^2\|
  \theta_t-\theta^\star  \|^2_{\sR^d},
\end{split}
\end{align}
where the last inequality used Condition \eqref{eq:stability}.
The desired inequality follows from the standard discrete   Gronwall lemma;  
 see e.g., \cite[Proposition 3.1]{emmrich1999discrete}.
 \end{proof}

\bibliographystyle{siam}
\bibliography{rl.bib}

\section*{Statements \& Declarations}

\subsection*{Funding}

Authors acknowledge the support of the UKRI Prosperity Partnership Scheme (FAIR) under EPSRC Grant EP/V056883/1 and the Alan Turing Institute. 

\subsection*{Competing Interests}

The authors have no relevant financial or non-financial interests to disclose.

\subsection*{Author Contributions}

All authors have contributed equally to this study.

\end{document}